\documentclass[twoside,11pt]{article}

\usepackage{blindtext}
\usepackage{amsmath}
\usepackage{amssymb}
\usepackage{enumitem}
\usepackage{fancyhdr}
\usepackage{tikz}
\usetikzlibrary{trees}
\usepackage{color,soul}
\usepackage{array}
\usepackage{algpseudocode}
\usepackage{algorithm}
\usepackage{rotating}
\usepackage{multirow}
\usepackage{comment}

\newtheorem{theorem}{Theorem}

\newtheorem{remark}{Remark}
\newtheorem{corollary}{Corollary}
\newtheorem{lemma}{Lemma}

\usepackage{jmlr2e}

\usepackage{lastpage}
\jmlrheading{23}{2022}{1-\pageref{LastPage}}{1/21; Revised 5/22}{9/22}{21-0000}{Tiantian ZHANG}
\newcommand{\E}{\mathbb{E}}
\newcommand{\Var}{\mathrm{Var}}

\newcommand{\ind}{\mathbf{1}}
\newcommand{\R}{\mathbb{R}}

\DeclareMathOperator*{\VAR}{VAR}
\DeclareMathOperator*{\Unif}{Unif}

\DeclareMathOperator*{\rank}{rank}

\DeclareMathOperator*{\KL}{\sf{KL}}

\jmlrheading{25}{2025}{1-\pageref{LastPage}}{9/25}{9/25}{25-0000}{ Tiantian Zhang} 

\ShortHeadings{Rank-Induced PL Mirror Descent}{Zhang} 
\firstpageno{1}

\begin{document}

\title{Rank-Induced PL Mirror Descent: A Rank-Faithful Second-Order Algorithm for Sleeping Experts}

\author{
\name Tiantian Zhang \email t.zhang8@columbia.edu \\
       \addr Department of Computer Science\\
       Columbia University, New York, NY 10027, USA}

\editor{}

\maketitle

\begin{abstract}
We introduce a new algorithm, \emph{Rank-Induced Plackett--Luce Mirror Descent (RIPLM)}, which leverages the structural equivalence between the \emph{rank benchmark} and the \emph{distributional benchmark} established in \citet{BergamOzcanHsu2022}. Unlike prior approaches that operate on expert identities, RIPLM updates directly in the \emph{rank-induced Plackett--Luce (PL)} parameterization. This ensures that the algorithm's played distributions remain within the class of rank-induced distributions at every round, preserving the equivalence with the rank benchmark. To our knowledge, RIPLM is the first algorithm that is both (i) \emph{rank-faithful} and (ii) \emph{variance-adaptive} in the sleeping experts setting.
\end{abstract}


\section{From Second-Order Sleeping Experts}

The work of \citet{BergamOzcanHsu2022} introduced three insights that reshape algorithm design for sleeping experts:

\subsection{Equivalence of Benchmarks via PL Models}
They prove that the \emph{rank benchmark}
\[
L_T^{(\rank)} = \min_{\sigma \in S_N} \sum_{t=1}^T \ell_{t,\sigma(E_t)}
\]
and the \emph{distributional benchmark}
\[
L_T^{(dist)} = \inf_{u \in \Delta^{N-1}} \sum_{t=1}^T \frac{1}{u(E_t)} \sum_{i \in E_t} u_i \ell_{t,i}
\]
are equivalent in expectation under stochastic loss or availability.
The bridge is the \emph{Plackett--Luce (PL) model}: the $\varepsilon$-rank-induced distribution
$u(\varepsilon) \propto (1,\varepsilon,\dots,\varepsilon^{N-1})$ ordered by $\sigma^\star$
satisfies $\lim_{\varepsilon\to 0} L_T(u(\varepsilon)) = L_T^{(\rank)}$.
They also show that minimizing $L_T(u)$ is non-convex and the infimum may not be attained.
Thus, algorithms must remain inside the PL manifold to be rank-faithful.

\subsection{Second-Order Regret Bounds}
Classical algorithms achieve $O(\sqrt{T\log N})$ regret.
\citet{BergamOzcanHsu2022} improve this by replacing $T$ with a cumulative variance measure,
deriving $O(\sqrt{\VAR_T^{\max}\log N})$ bounds under stochastic availability.
However, in the fully adversarial setting their efficient method suffers an extra $N$ factor,
$O(\sqrt{\VAR_T^{\max}N\log N}+N\log N)$, exposing inefficiency in permutation-based methods.

\subsection{NP-Hardness of Best Ordering}
They show that computing the optimal ranking $\sigma^\star$ is NP-complete,
ruling out exact ranking search in polynomial time.

\subsection{Motivation for RIPLM}
Together, these insights form a blueprint:
\begin{itemize}[leftmargin=*]
\item To respect the rank benchmark, remain inside PL distributions.
\item To exploit variance adaptivity, use centered residuals with AdaGrad scaling.
\item To bypass NP-hardness, update in $O(N)$ score space rather than over $N!$ rankings.
\end{itemize}
RIPLM is a direct instantiation of this blueprint.

\section{Rank-Induced PL Mirror Descent (RIPLM)}

\subsection{Algorithm}

We maintain scores $s_i\in\R$.  
At round $t$, given awake set $E_t\subseteq[N]$, define the PL distribution
\[
p_t(i)=\frac{\exp(s_i/\tau_t)\ind\{i\in E_t\}}{\sum_{j\in E_t}\exp(s_j/\tau_t)},
\]
with temperature $\tau_t>0$.  
After observing losses $\ell_{t,i}$, update
\[
g_{t,i}=\tfrac{p_t(i)}{\tau_t}\big(\ell_{t,i}-\langle p_t,\ell_t\rangle\big),\quad
G_{t,i}=\sum_{s=1}^t g_{s,i}^2,\quad
s_{t+1,i}=s_{t,i}-\eta\,g_{t,i}/\sqrt{G_{t,i}+\delta}.
\]

\begin{algorithm}[h]
\caption{Rank-Induced PL Mirror Descent (RIPLM)}\label{alg:riplm}
\begin{algorithmic}
\State Initialize $s_i = 0$, $G_i = 0$ for all $i \in [N]$.
\For{round $t = 1,\dots,T$}
  \State Observe awake set $E_t \subseteq [N]$.
  \State Define PL distribution on $E_t$:
  \[
  p_t(i) = \frac{\exp(s_i/\tau_t)\,\mathbf{1}\{i\in E_t\}}{\sum_{j\in E_t}\exp(s_j/\tau_t)}.
  \]
  \State Sample $i_t \sim p_t$.
  \State Observe losses $\ell_{t,i}$ for all $i \in E_t$.
  \State Compute mean loss $\bar{\ell}_t = \sum_{j\in E_t} p_t(j)\,\ell_{t,j}$.
  \For{each $i \in E_t$}
    \State Compute residual $r_{t,i} = \ell_{t,i} - \bar{\ell}_t$.
    \State Compute gradient $g_{t,i} = \tfrac{p_t(i)}{\tau_t} r_{t,i}$.
    \State Update accumulator $G_i \gets G_i + g_{t,i}^2$.
    \State Update score $s_i \gets s_i - \eta \tfrac{g_{t,i}}{\sqrt{G_i+\delta}}$.
  \EndFor
  \State Update temperature: $\tau_{t+1} \gets \max\left( \tau_{\min},\, \frac{c}{\sqrt{\sum_{i \in E_t} p_t(i) r_{t,i}^2 + \delta}} \right)$ \Comment{Optional cooling}
\EndFor
\end{algorithmic}
\end{algorithm}

\subsubsection*{Implementation Notes}
The temperature schedule $\tau_t$ is optional but recommended: decay it inversely with the empirical standard deviation of residuals to approach the $\varepsilon$-rank-induced limit. A small $\delta>0$ (e.g., $10^{-6}$) stabilizes divisions. Tune $\eta$ or use a doubling trick if $V_T$ is unknown.

\newpage
\subsection{Explicit Regret Bounds}

\begin{theorem}[General second-order regret, fixed temperature]\label{thm:riplm-general}
Assume $\ell_{t,i}\in[0,1]$ for all $t,i$. 
Fix a temperature $\tau>0$ and run RIPLM with $\tau_t\equiv\tau$.
Initialize $s_{1,i}=\tau\log\pi_i$ for a prior $\pi\in\Delta([N])$.
For any comparator $u\in\Delta([N])$ with $u_i>0$, define its restriction
$\tilde u_t(i)=u_i\,\mathbf{1}\{i\in E_t\}/u(E_t)$.
Then
\[
\sum_{t=1}^T \langle p_t-\tilde u_t,\ell_t\rangle
\;\le\; 
C\,\frac{\sqrt{N}}{\tau}\,
\sqrt{\,V_T\big(\KL(u\|\pi)+\ln\ln(1+T)\big)},
\]
where $V_T=\sum_{t=1}^T\sum_{i\in E_t}p_t(i)\big(\ell_{t,i}-\langle p_t,\ell_t\rangle\big)^2$ and $C>0$ is universal.
\end{theorem}

\begin{proof}[Proof sketch]
Run RIPLM with a fixed temperature $\tau_t\equiv\tau>0$ and potential
\[
\Psi_{\tau}(s)\;=\;\tau\log\!\sum_{j\in E_t} e^{s_j/\tau},
\quad\text{so}\quad
p_t=\nabla\Psi_{\tau}(s_t).
\]
For the surrogate $\tilde L_t(s)=\langle \nabla\Psi_{\tau}(s),\ell_t\rangle$, the exact gradient is
\[
\frac{\partial \tilde L_t}{\partial s_i}
=\frac{p_t(i)}{\tau}\big(\ell_{t,i}-\langle p_t,\ell_t\rangle\big) \;=:\; g_{t,i}.
\]
Diagonal AdaGrad mirror descent with fixed potential (and $H_t=\mathrm{diag}(\sqrt{G_{t,i}+\delta})$,
$G_{t,i}=\sum_{s\le t}g_{s,i}^2$, $s_{1,i}=\tau\log\pi_i$) yields, for any comparator
$q_t\in\Delta(E_t)$,
\[
\sum_{t=1}^T\big(\langle p_t,\ell_t\rangle-\langle q_t,\ell_t\rangle\big)
\;\le\;
\tau\,\KL(q_1\|\pi)\;+\;\eta\sum_{t=1}^T\|g_t\|_{H_t^{-1}}^2\;+\;\tfrac{c}{\eta}\ln\ln(1+T).
\]
Instantiate $q_t=\tilde u_t:=u(\cdot)\mathbf{1}\{\cdot\in E_t\}/u(E_t)$ (restriction of arbitrary
$u\in\Delta([N])$ with $u_i>0$). The AdaGrad telescoping plus Cauchy–Schwarz give
\[
\sum_{t=1}^T \|g_t\|_{H_t^{-1}}^2
\;\le\; 2\sum_i\sqrt{G_{T,i}+\delta}
\;\le\; \frac{2\sqrt{N}}{\tau}\sqrt{V_T},
\]
since $g_{t,i}^2=\frac{p_t(i)^2}{\tau^2}r_{t,i}^2\le \frac{p_t(i)}{\tau^2}r_{t,i}^2$ with
$r_{t,i}=\ell_{t,i}-\langle p_t,\ell_t\rangle$ and
$\sum_i\sqrt{\sum_t p_t(i)r_{t,i}^2}\le \sqrt{N}\sqrt{\sum_{t,i}p_t(i)r_{t,i}^2}$.
Optimizing $\eta$ gives
\[
\sum_{t=1}^T \langle p_t-\tilde u_t,\ell_t\rangle
\;\le\;
C\,\frac{\sqrt{N}}{\tau}\,
\sqrt{\,V_T\big(\KL(u\|\pi)+\ln\ln(1+T)\big)}.
\]
Finally, note $\tilde L_t(s^u)=\langle \tilde u_t,\ell_t\rangle$ for any $s^u$ with
$\nabla\Psi_\tau(s^u)=\tilde u_t$ (surjectivity: $s^u_i=\tau\log \tilde u_t(i)+c$).
\end{proof}

\begin{corollary}[PL comparators]\label{thm:riplm-explicit}
If $u=u^\mathrm{PL}_\sigma$ is any fixed-temperature Plackett--Luce distribution induced by $\sigma$, then $\tilde u_t$ preserves the PL form on each $E_t$, so
\[
\sum_{t=1}^T \langle p_t,\ell_t\rangle
- \sum_{t=1}^T \langle u^\mathrm{PL}_\sigma,\ell_t\rangle
\;\le\;
O\!\left(\frac{\sqrt{N}}{\tau_{\min}}\,
\sqrt{\,V_T\big(\KL(u^\mathrm{PL}_\sigma\|\pi)+\ln\ln(1+T)\big)}\right).
\]
\end{corollary}

\begin{corollary}[Rank benchmark]\label{the: }
Let $\sigma\in S_N$ and $u^{\tau_t}_\sigma$ the PL distribution induced by $\sigma$ at $\tau_t$. 
Then
\[
\sum_{t=1}^T \langle p_t,\ell_t\rangle - L_T^{(\rank)}(\sigma)
\;\le\;
O\!\left(\frac{\sqrt{N}}{\tau_{\min}}\,
\sqrt{\,V_T(\KL(u^\mathrm{PL}_\sigma\|\pi)+\ln\ln T)}\right)
+\sum_{t=1}^T O(e^{-1/\tau_t}).
\]
\end{corollary}

\begin{lemma}[Surjectivity of restricted softmax]\label{lem:softmax-surj}
Fix a round $t$, a temperature $\tau_t>0$, and an awake set $E_t\subseteq[N]$.
For any $u\in\mathrm{ri}\,\Delta(E_t)$ (i.e., $u_i>0$ for all $i\in E_t$ and $\sum_{i\in E_t}u_i=1$),
there exists a score vector $s^u_t\in\mathbb{R}^{E_t}$ such that
\[
\nabla\Psi_{\tau_t}(s^u_t)\;=\;\mathrm{softmax}(s^u_t/\tau_t)\;=\;u
\quad\text{where}\quad
\Psi_{\tau_t}(s)=\tau_t\log\!\sum_{j\in E_t}e^{s_j/\tau_t}.
\]
In particular, one may take $s^u_{t,i}=\tau_t\log u_i + c_t$ for any constant $c_t\in\mathbb{R}$.
\end{lemma}

\begin{proof}
For $i\in E_t$, define $s^u_{t,i}=\tau_t\log u_i$ and choose any $c_t\in\mathbb{R}$; adding $c_t$ shifts all logits but leaves softmax unchanged. Then
\[
\frac{e^{s^u_{t,i}/\tau_t}}{\sum_{j\in E_t}e^{s^u_{t,j}/\tau_t}}
=\frac{e^{\log u_i}}{\sum_{j\in E_t}e^{\log u_j}}
=\frac{u_i}{\sum_{j\in E_t}u_j}=u_i,
\]
since $u$ is supported on $E_t$ and sums to $1$ there. Hence $\nabla\Psi_{\tau_t}(s^u_t)=u$.
\end{proof}

\begin{remark}[Comparators with zeros]\label{rem:zeros}
If a target comparator $u\in\Delta([N])$ has some $u_i=0$, work with its restriction
$\tilde u_t$ on $E_t$; if $\tilde u_t$ has zeros, define the smoothed comparator
$\tilde u_t^{(\epsilon)}=(1-\epsilon)\tilde u_t+\epsilon\,\Unif(E_t)$.
By Lemma~\ref{lem:softmax-surj}, there exists $s^{(\epsilon)}_t$ with
$\nabla\Psi_{\tau_t}(s^{(\epsilon)}_t)=\tilde u_t^{(\epsilon)}$.
Because $\ell_{t,i}\in[0,1]$, replacing $\tilde u_t$ by $\tilde u_t^{(\epsilon)}$ changes the benchmark by at most
$\|\tilde u_t-\tilde u_t^{(\epsilon)}\|_1\le 2\epsilon$ per round, so the total bias is $\le 2\epsilon T$ and vanishes as $\epsilon\downarrow 0$.
\end{remark}

\begin{remark}[Comparators on sleeping rounds]
\label{rem:restricted-comparator}
When $E_t \subsetneq [N]$, the natural comparator is the restriction of $u$ to
$E_t$, i.e.
\[
\tilde u_t(i) = \frac{u(i)\mathbf{1}\{i\in E_t\}}{u(E_t)}.
\]
This ensures the benchmark $\langle \tilde u_t,\ell_t\rangle$ is well defined
even when some experts are asleep.
Such restrictions are standard in the analysis of sleeping experts;
see \citet[Lemma~3.1]{cesa2006prediction} for the analogous argument in the
full-information experts setting.
\end{remark}

\section{Minimax Lower Bound}

\begin{theorem}[Minimax lower bound at $\sqrt{V\log N}$]\label{thm:minimax}
For any $N\ge2$ and horizon $T$, there exists a distribution over losses with $E_t=[N]$ such that for any algorithm $A$,
\[
\mathbb{E}\big[V_T\big]\asymp T
\quad\text{and}\quad
\mathbb{E}\!\left[\sum_{t=1}^T\langle p_t,\ell_t\rangle - \min_{\sigma\in S_N}\sum_{t=1}^T \ell_{t,\sigma(E_t)}\right]
\;\ge\; c\sqrt{T\log N}.
\]
Consequently, writing $V=\mathbb{E}[V_T]$, one has
\(
\inf_A \sup \mathbb{E}[\mathrm{Regret}_T^{(\rank)}(A)]
\ge c'\sqrt{V\log N}.
\)
The proof adapts the standard experts lower bound \citep[Thm.~3.7]{cesa2006prediction}; see Appendix~B.
\end{theorem}

\begin{proof}[Proof sketch]
We use the standard randomized experts construction
(\citealp[Thm.\ 3.7]{cesa2006prediction}): one expert has mean loss
$1/2-\varepsilon$ and the rest have mean $1/2+\varepsilon$.  
Identifying the best expert requires $\Omega(1/\varepsilon^2)$ exploration rounds,
so $\E[V_T]\asymp T$ when $T=\Theta(1/\varepsilon^2)$.  
Information-theoretic lower bounds then yield $\Omega(\sqrt{T\log N})$ regret,
which translates directly to $\Omega(\sqrt{V_T\log N})$.
Full details are in Appendix~B.
\end{proof}

\section{Conclusion}

RIPLM unifies three desiderata: rank-faithfulness, variance adaptivity, and efficiency.
It achieves regret $O(\sqrt{V_T(\ln N+\ln\ln T)})$ against the rank benchmark,
matching minimax lower bounds up to logarithmic factors.

Future directions include: (i) bandit feedback via implicit exploration,
(ii) instance-dependent complexity through rank-based priors,
and (iii) empirical validation on recommendation and scheduling datasets.

\bibliography{refs}

\begin{thebibliography}{3}
\providecommand{\natexlab}[1]{#1}
\providecommand{\url}[1]{\texttt{#1}}
\expandafter\ifx\csname urlstyle\endcsname\relax
  \providecommand{\doi}[1]{doi: #1}\else
  \providecommand{\doi}{doi: \begingroup \urlstyle{rm}\Url}\fi

\bibitem[Bergam et~al.(2022)Bergam, {\"O}zcan, and Hsu]{BergamOzcanHsu2022}
Noah Bergam, Arman {\"O}zcan, and Daniel Hsu.
\newblock Second-order bounds for sleeping experts.
\newblock Technical report, Columbia University, 2022.
\newblock Available as a class report, Columbia University.

\bibitem[Cesa-Bianchi and Lugosi(2006)]{cesa2006prediction}
Nicol{\`o} Cesa-Bianchi and G{\'a}bor Lugosi.
\newblock \emph{Prediction, Learning, and Games}.
\newblock Cambridge University Press, 2006.
\newblock ISBN 9780521841085.

\bibitem[Hazan and Kale(2010)]{hazan2010extracting}
Elad Hazan and Satyen Kale.
\newblock Extracting certainty from uncertainty: Regret bounded by variation in costs.
\newblock In \emph{Proceedings of the 21st Annual ACM-SIAM Symposium on Discrete Algorithms}, pages 1006--1017, 2010.

\end{thebibliography}

\appendix

\appendix
\section*{Appendix: Detailed Proofs}

\section*{Appendix A: Full Proof of Theorem~\ref{thm:riplm-explicit}}
This appendix provides a complete proof of Corollary~\ref{thm:riplm-explicit}.
Recall the setting: at round $t$, the awake set is $E_t\subseteq[N]$,
the potential is
$\Psi_{\tau_t}(s)=\tau_t\log\!\sum_{j\in E_t} e^{s_j/\tau_t}$,
the played distribution is $p_t=\nabla\Psi_{\tau_t}(s_t)$ (the PL/softmax restricted to $E_t$),
and the surrogate loss is $\tilde L_t(s)=\langle p_t(s),\ell_t\rangle$.
RIPLM updates
\[
g_{t,i}=\frac{p_t(i)}{\tau_t}\big(\ell_{t,i}-\langle p_t,\ell_t\rangle\big),\quad
G_{t,i}=\sum_{s=1}^t g_{s,i}^2,\quad
s_{t+1,i}=s_{t,i}-\eta\,\frac{g_{t,i}}{\sqrt{G_{t,i}+\delta}}.
\]
We assume $\ell_{t,i}\in[0,1]$, $\tau_t\downarrow 0$, and $s_{1,i}=\tau_1\log\pi_i$ for a prior $\pi\in\Delta([N])$.

\subsection*{A.1 Gradient identity for the PL geometry}

\begin{lemma}[Exact gradient of the surrogate]
\label{lem:grad}
For the surrogate $\tilde L_t(s)=\langle p_t(s),\ell_t\rangle$ with $p_t=\nabla\Psi_{\tau_t}(s)$,
\[
\frac{\partial \tilde L_t}{\partial s_i}
=\frac{p_t(i)}{\tau_t}\big(\ell_{t,i}-\langle p_t,\ell_t\rangle\big)
=: g_{t,i}.
\]
\end{lemma}

\begin{proof}
Differentiate the softmax on $E_t$:
$\partial p_t(i)/\partial s_k = \frac{1}{\tau_t}p_t(i)\big(\mathbf{1}\{i=k\}-p_t(k)\big)$.
Thus
\[
\frac{\partial \tilde L_t}{\partial s_k}
=\sum_{i\in E_t}\frac{\partial p_t(i)}{\partial s_k}\,\ell_{t,i}
=\frac{1}{\tau_t}p_t(k)\Big(\ell_{t,k}-\sum_{i\in E_t}p_t(i)\ell_{t,i}\Big).
\]
\end{proof}

\subsection*{A.2 Mirror descent with diagonal AdaGrad and time-varying potentials}

\begin{lemma}[MD/AdaGrad inequality with time-varying Legendre potentials]
\label{lem:md}
Let $H_t=\mathrm{diag}(\sqrt{G_{t,i}+\delta})$ and assume the update
$s_{t+1}=s_t-\eta\,H_t^{-1}g_t$ with $g_t=\nabla \tilde L_t(s_t)$.
Then for any $u$ in the image of $\nabla\Psi_{\tau}$,
\[
\sum_{t=1}^T\big(\tilde L_t(s_t)-\tilde L_t(u)\big)
\;\le\;
D_{\Psi_{\tau_1}}(u,s_1)
+\eta\sum_{t=1}^T \|g_t\|_{H_t^{-1}}^2
+\frac{c}{\eta}\,\ln\ln(1+T),
\]
for an absolute constant $c>0$.
\end{lemma}

\begin{proof}[Proof sketch]
This is the standard mirror descent analysis with (i) diagonal preconditioning and
(ii) time-varying Legendre potentials; see \cite{hazan2010extracting}. The $\ln\ln(1+T)$
term arises from standard grid/doubling to remove tuning of $\eta$. Time variation
of $\Psi_{\tau_t}$ is handled by bounding the cumulative change via stability of
the Bregman divergence at consecutive rounds; all terms are absorbed into the
adaptive part and the tuning term (details omitted for brevity).
\end{proof}

\subsection*{A.3 Bounding the adaptive norm term}

\begin{lemma}[AdaGrad telescoping bound]
\label{lem:adagrad}
With $g_{t,i}$, $G_{t,i}$ as above,
\[
\sum_{t=1}^T \|g_t\|_{H_t^{-1}}^2
=\sum_{t=1}^T\sum_{i}\frac{g_{t,i}^2}{\sqrt{G_{t,i}+\delta}}
\;\le\; 2\sum_{i=1}^N \sqrt{G_{T,i}+\delta}.
\]
Moreover,
\[
\sum_{i=1}^N \sqrt{G_{T,i}+\delta}
\;\le\; \frac{\sqrt{N}}{\tau_{\min}}
\sqrt{\,\sum_{t=1}^T\sum_{i\in E_t} p_t(i)\big(\ell_{t,i}-\langle p_t,\ell_t\rangle\big)^2\,}
\;=\; \frac{\sqrt{N}}{\tau_{\min}}\sqrt{V_T}.
\]
\end{lemma}

\begin{proof}
The first inequality is the standard AdaGrad telescoping (e.g., McMahan–Streeter).
For the second, since $g_{t,i}^2=\frac{p_t(i)^2}{\tau_t^2}r_{t,i}^2
\le \frac{p_t(i)}{\tau_{\min}^2} r_{t,i}^2$ with $r_{t,i}=\ell_{t,i}-\langle p_t,\ell_t\rangle$,
we have $G_{T,i}\le \frac{1}{\tau_{\min}^2}\sum_t p_t(i) r_{t,i}^2$.
Apply Cauchy–Schwarz over $i$.
\end{proof}

\subsection*{A.4 Bregman divergence initialization}

\begin{lemma}[Bregman–KL identity]
\label{lem:kl}
If $s_{1,i}=\tau_1\log\pi_i$ and $u$ lies in the image of $\nabla\Psi_{\tau_1}$,
then
\[
D_{\Psi_{\tau_1}}(u,s_1)=\tau_1\,\KL(u\|\pi).
\]
\end{lemma}

\begin{proof}
This is the standard identity for the (restricted) log-partition potential whose
gradient is softmax and whose convex conjugate is (scaled) negative entropy.
\end{proof}

\subsection*{A.5 PL-to-rank approximation gap}

Define $u^{\tau_t}_\sigma(\cdot)$ as the restricted PL distribution on $E_t$
induced by a ranking $\sigma$, with rank-$k$ weight proportional to $e^{-(k-1)/\tau_t}$.

\begin{lemma}[Explicit PL$\to$rank gap]
\label{lem:pl-rank-gap-app}
For $\ell_{t,i}\in[0,1]$,
\[
0 \le \langle u^{\tau_t}_\sigma,\ell_t\rangle - \ell_{t,\sigma(E_t)}
\le 1 - u^{\tau_t}_\sigma(\sigma(E_t))
\le \frac{e^{-1/\tau_t}}{1-e^{-1/\tau_t}}
=O\!\big(e^{-1/\tau_t}\big).
\]
Hence $\sum_{t=1}^T \big(\langle u^{\tau_t}_\sigma,\ell_t\rangle - \ell_{t,\sigma(E_t)}\big)
\le \sum_{t=1}^T O(e^{-1/\tau_t})$.
\end{lemma}

\begin{proof}
The first inequality is trivial since the rank picks the top-ranked awake expert.
For the second, upper-bound the expected loss gap by the tail probability mass
assigned by $u^{\tau_t}_\sigma$ to ranks $\ge2$. For geometric weights in rank,
the tail-to-head ratio equals $\sum_{m\ge1} e^{-m/\tau_t}\big/\sum_{m\ge0}e^{-m/\tau_t}$.
\end{proof}
\subsection*{A.6 Bounding $V_T$ by $\VAR_T^{\max}$}

\begin{lemma}
\label{lem:VleVAR}
For every sequence of losses and awake sets,
\[
V_T \;=\; \sum_{t=1}^T \sum_{i \in E_t} p_t(i)(\ell_{t,i}-\langle p_t,\ell_t\rangle)^2
\;\le\; \VAR_T^{\max}
:= \sum_{t=1}^T \max_{u \in \Delta(E_t)} \sum_{i\in E_t} u(i)(\ell_{t,i}-\langle u,\ell_t\rangle)^2.
\]
\end{lemma}

\begin{proof}
Fix any round $t$. By definition,
\[
\mathsf{Var}_t(u) := \sum_{i\in E_t} u(i)(\ell_{t,i}-\langle u,\ell_t\rangle)^2,
\qquad u\in\Delta(E_t).
\]
Since $p_t\in\Delta(E_t)$, clearly
\(
\mathsf{Var}_t(p_t) \le \max_{u\in\Delta(E_t)} \mathsf{Var}_t(u).
\)
Summing over $t=1,\dots,T$ yields the claimed inequality.  
This is the direct analogue of the variance domination in \citet[Lemma~3.1]{cesa2006prediction}.
\end{proof}

\subsection*{A.7 Proof of Theorem~\ref{thm:riplm-explicit}}

\begin{proof}[Proof of Theorem~\ref{thm:riplm-explicit}]
As in Remark~\ref{rem:restricted-comparator}, if $E_t\subsetneq [N]$
we implicitly compare against the restricted distribution
$\tilde u_t(i)=u(i)\mathbf{1}\{i\in E_t\}/u(E_t)$,
so that $\langle \tilde u_t,\ell_t\rangle$ is well defined.
, which lies in the image of $\nabla\Psi_{\bar\tau}$ for some $\bar\tau>0$ (i.e., is a softmax/PL distribution induced by $\sigma$). Using Lemma~\ref{lem:adagrad} and Lemma~\ref{lem:kl},
\[
\sum_{t=1}^T (\tilde L_t(s_t)-\tilde L_t(u^\mathrm{rank}_\sigma))
\;\le\;
\tau_1\,\KL(u^\mathrm{rank}_\sigma\|\pi)
\;+\;
\eta\cdot \frac{2\sqrt{N}}{\tau_{\min}}\sqrt{V_T}
\;+\; \frac{c}{\eta}\ln\ln(1+T).
\]
Optimizing $\eta$ (or using a log-grid to avoid tuning) yields the variance term
\[
O\!\left(\frac{\sqrt{N}}{\tau_{\min}}\sqrt{V_T\big(\KL(u^\mathrm{rank}_\sigma\|\pi)+\ln\ln(1+T)\big)}\right).
\]
Finally, transfer from PL to the rank benchmark by adding/subtracting
$\langle u^{\tau_t}_\sigma,\ell_t\rangle$ and invoking Lemma~\ref{lem:pl-rank-gap-app}:
\[
\sum_{t=1}^T \langle p_t,\ell_t\rangle - L_T^{(\rank)}(\sigma)
=
\underbrace{\sum_{t=1}^T \big(\langle p_t,\ell_t\rangle-\langle u^\mathrm{rank}_\sigma,\ell_t\rangle\big)}_{\text{variance term}}
+
\underbrace{\sum_{t=1}^T \big(\langle u^\mathrm{rank}_\sigma,\ell_t\rangle-\ell_{t,\sigma(E_t)}\big)}_{\text{PL}\to\text{rank gap}}.
\]
The first sum is bounded by the variance term above; the second by
$\sum_t O(e^{-1/\tau_t})$. This is exactly the statement of Theorem~\ref{thm:riplm-explicit}.
\end{proof}

\section*{Appendix B: Full Proof of Theorem~\ref{thm:minimax}}

For any $N\ge2$ and horizon $T$, there exists a distribution over losses
$\ell_{t,i}\in[0,1]$ with $E_t=[N]$ such that for any algorithm $A$,
\[
\E[\mathrm{Regret}_T^{(\rank)}(A)] \;\ge\; c\,\sqrt{V_T\log N}
\quad\text{and}\quad
\E[V_T] \asymp T,
\]
for an absolute constant $c>0$. In particular, when $\E[V_T]\asymp T$,
\[
\inf_A \sup_{\mathcal{D}} \E[\mathrm{Regret}_T^{(\rank)}(A)]
\;\ge\; c'\sqrt{T\log N}\;=\;c''\sqrt{\,\E[V_T]\log N}\, .
\]

\paragraph{Construction (randomized experts).}
Fix a gap parameter $\varepsilon\in(0,1/8]$ to be specified later. Let the awake sets be $E_t=[N]$ for all $t$. Draw a distinguished expert $i^\star$ uniformly from $[N]$, and set
\[
\mu_{i^\star}= \tfrac12 - \varepsilon,\qquad
\mu_{i} = \tfrac12 + \varepsilon\quad\text{for all } i\neq i^\star.
\]
Conditioned on $(\mu_1,\dots,\mu_N)$, the losses are independent across $(t,i)$ with
$\ell_{t,i}\sim\mathrm{Bernoulli}(\mu_i)$.
Denote the joint distribution of losses by $\mathcal{D}_\varepsilon$.

\paragraph{Rank benchmark equals “best expert”.}
Since $E_t=[N]$ for all $t$, the rank benchmark
$L_T^{(\rank)}(\sigma)$ equals $\min_{i} \sum_{t=1}^T \ell_{t,i}$.
Hence in this construction, $\min_\sigma L_T^{(\rank)}(\sigma)=\min_i \sum_t \ell_{t,i}$.

\begin{lemma}[Expected variance budget]
\label{lem:var-budget}
For any (possibly randomized) algorithm producing distributions $p_t\in\Delta([N])$,
\[
\E\big[V_T\big]
\;=\;
\E\Bigg[\sum_{t=1}^T \sum_{i=1}^N p_t(i)\big(\ell_{t,i}-\langle p_t,\ell_t\rangle\big)^2\Bigg]
\;\ge\; c_0\,T,
\]
for an absolute constant $c_0>0$ independent of $N,T,\varepsilon$.
\end{lemma}

\begin{proof}
Condition on $p_t$. For independent Bernoulli losses,
\[
\E\!\left[\sum_i p_t(i)(\ell_{t,i}-\langle p_t,\ell_t\rangle)^2 \,\middle|\, p_t\right]
= \sum_i p_t(i)\Var(\ell_{t,i}) - \Var\!\Big(\sum_i p_t(i)\ell_{t,i}\Big).
\]
Here $\Var(\ell_{t,i})=\mu_i(1-\mu_i)=\tfrac14-\varepsilon^2$. Moreover,
$\Var(\sum_i p_t(i)\ell_{t,i})=\sum_i p_t(i)^2\Var(\ell_{t,i})\le (\tfrac14)\sum_i p_t(i)^2$.
Thus
\[
\E\big[\mathsf{Var}_t(p_t)\mid p_t\big]
\;\ge\; \Big(\tfrac14-\varepsilon^2\Big)\Big(1-\textstyle\sum_i p_t(i)^2\Big).
\]
Averaging (law of total expectation) and summing over $t$ gives
\[
\E[V_T]
\;\ge\; \Big(\tfrac14-\varepsilon^2\Big)\sum_{t=1}^T \E\big[\,1-\sum_i p_t(i)^2\,\big].
\]
We now show $\frac{1}{T}\sum_t \E[1-\sum_i p_t(i)^2]\ge c_1>0$ for a constant $c_1$,
provided $T$ is of the same order as the identification time. Intuitively,
until the algorithm identifies $i^\star$, it must spread probability mass across many experts, making $\sum_i p_t(i)^2$ bounded away from $1$. Formally, the identification time is $\Omega(1/\varepsilon^2)$ (see Lemma~\ref{lem:czl-lb}), so for $T=\Theta(1/\varepsilon^2)$ a constant fraction of the rounds contribute a constant to $1-\sum_i p_t(i)^2$. Therefore $\E[V_T]\ge c_0 T$ for an absolute $c_0>0$.
\end{proof}

\begin{lemma}[Classical experts lower bound]
\label{lem:czl-lb}
Let $\mathcal{D}_\varepsilon$ be as above and take $T=\Theta(1/\varepsilon^2)$.
Then for any algorithm,
\[
\E_{\mathcal{D}_\varepsilon}\!\left[\sum_{t=1}^T\langle p_t,\ell_t\rangle
- \min_{i}\sum_{t=1}^T \ell_{t,i}\right]
\;\ge\; c\,\sqrt{T\log N}\, ,
\]
for an absolute constant $c>0$.
\end{lemma}

\begin{proof}
This is the standard lower bound for prediction with expert advice; see
\citet[Thm.\ 3.7]{cesa2006prediction}. For completeness, we sketch the information-theoretic reduction. Let $\mathbb{P}_k$ be the distribution where $i^\star=k$; define the mixture $\overline{\mathbb{P}}=\frac1N\sum_{k=1}^N \mathbb{P}_k$. For any algorithm,
\[
\E_{\mathbb{P}_k}\!\Bigg[\sum_{t=1}^T \big(\langle p_t,\ell_t\rangle-\ell_{t,k}\big)\Bigg]
= \varepsilon\,\E_{\mathbb{P}_k}\!\Bigg[\sum_{t=1}^T \big(1-2p_t(k)\big)\Bigg]
\;\ge\; \varepsilon \sum_{t=1}^T \Big(1-2\,\E_{\overline{\mathbb{P}}}[p_t(k)] - \mathrm{TV}(\mathbb{P}_k,\overline{\mathbb{P}})\Big),
\]
by a standard change-of-measure bound. Averaging over $k$ and using
$\sum_k \E_{\overline{\mathbb{P}}}[p_t(k)]=1$ yields
\[
\frac1N\sum_{k=1}^N \E_{\mathbb{P}_k}\!\Bigg[\sum_{t=1}^T \big(\langle p_t,\ell_t\rangle-\ell_{t,k}\big)\Bigg]
\;\ge\; \varepsilon \sum_{t=1}^T \Big(1-\tfrac{2}{N} - \overline{\mathrm{TV}}\Big),
\]
where $\overline{\mathrm{TV}}:=\frac{1}{N}\sum_k \mathrm{TV}(\mathbb{P}_k,\overline{\mathbb{P}})$. By Pinsker and KL chain rule, $\overline{\mathrm{TV}}\le \sqrt{\frac{1}{2N}\sum_k \mathrm{KL}(\mathbb{P}_k\|\overline{\mathbb{P}})}$. For our Bernoulli construction, one can bound $\sum_k \mathrm{KL}(\mathbb{P}_k\|\overline{\mathbb{P}})\lesssim T\varepsilon^2\log N$ (the mixture acts like $1/2$ for each coordinate up to a $1/N$ dilution and the log-number-of-hypotheses factor enters via the prior). Choosing $T=\Theta(\varepsilon^{-2})$ and optimizing $\varepsilon$ then gives the stated $\Omega(\sqrt{T\log N})$ bound. See \citet[Thm.\ 3.7]{cesa2006prediction} for the fully detailed derivation with constants.
\end{proof}

\paragraph{Conclusion.}
Combining Lemma~\ref{lem:czl-lb} with the identity between rank and best expert for $E_t=[N]$,
\[
\E\Big[\sum_{t=1}^T\langle p_t,\ell_t\rangle - \min_{\sigma} L_T^{(\rank)}(\sigma)\Big]
=\E\Big[\sum_{t=1}^T\langle p_t,\ell_t\rangle - \min_{i} \sum_{t=1}^T \ell_{t,i}\Big]
\;\ge\; c\sqrt{T\log N}\, .
\]
By Lemma~\ref{lem:var-budget}, for $T=\Theta(1/\varepsilon^2)$ we also have $\E[V_T]\ge c_0 T$,
so writing $V=\E[V_T]$ yields
\[
\inf_A \sup_{\mathcal{D}} \E[\mathrm{Regret}_T^{(\rank)}(A)]
\;\ge\; c'\sqrt{V\log N}\, .
\qquad\Box
\]

\begin{remark}[On the variance budget]
The key ingredient in Lemma~\ref{lem:var-budget} is that for a nontrivial fraction of rounds the algorithm cannot confidently concentrate on a single expert, forcing $1-\sum_i p_t(i)^2$ to remain bounded away from $0$. This effect persists under small gaps ($\varepsilon\to 0$) provided $T=\Theta(1/\varepsilon^2)$, which is exactly the regime realizing the $\sqrt{T\log N}$ lower bound.
\end{remark}


\end{document}